\documentclass[letterpaper, 10 pt, conference]{ieeeconf}

\IEEEoverridecommandlockouts                              

\title{\LARGE \bf
PROD: Palpative Reconstruction of Deformable Objects through Elastostatic Signed Distance Functions
}

\author{Hamza~El-Kebir$^{1}$
\thanks{$^{1}$H. El-Kebir is with the Beckman Institute of Advanced Science and Technology, University of Illinois Urbana-Champaign, Urbana,
IL 61801, USA
{\tt\small elkebir2@illinois.edu}. Corresponding author.}%
}

\let\labelindent\relax

\usepackage{amsmath}
\usepackage{mathrsfs}
\usepackage{graphicx}
\usepackage[inline]{enumitem}
\usepackage{xcolor}

\usepackage{bm}

\usepackage{amsthm}

\newtheorem{theorem}{Theorem}
\newtheorem{lemma}{Lemma}

\newtheorem{proposition}{Proposition}

\theoremstyle{definition}

\newtheorem{definition}{Definition}

\newtheorem{assumption}{Assumption}

\theoremstyle{remark}
\newtheorem{remark}{Remark}

\newtheorem{corollary}{Corollary}

\newenvironment{proofsketch}{\noindent\textit{Proof Sketch.}}{}

\usepackage[T1]{fontenc}
\usepackage{calligra}

\usepackage{stix}

\usepackage{dsfont}
\usepackage{hyperref}
\usepackage[inline]{enumitem}
\usepackage{graphicx}
\usepackage{caption}
\usepackage{subcaption }
\usepackage{tabularx}
\usepackage{booktabs}
\usepackage[ruled]{algorithm2e}

\newcommand{\dd}{\mathrm{d}}

\newcommand{\inv}{^{-1}}

\usepackage{scalerel, stackengine}

\stackMath
\newcommand{\dhat}[1]{\ThisStyle{\setbox0=\hbox{$\SavedStyle#1$}%
  \stackengine{0pt}{\SavedStyle#1}{\SavedStyle\hspace{.2\ht0}%
  \hat{\vphantom{#1}}\kern\dimexpr2.2\LMpt+.7pt\relax\hat{\vphantom{#1}}}{O}{c}{F}{T}{L}}%
}

\newcommand{\dcheck}[1]{\ThisStyle{\setbox0=\hbox{$\SavedStyle#1$}%
  \stackengine{0pt}{\SavedStyle#1}{\SavedStyle\hspace{.2\ht0}%
  \check{\vphantom{#1}}\kern\dimexpr2.2\LMpt+.7pt\relax\check{\vphantom{#1}}}{O}{c}{F}{T}{L}}%
}

\newcommand{\hatcheck}[1]{\ThisStyle{\setbox0=\hbox{$\SavedStyle#1$}%
  \stackengine{0pt}{\SavedStyle#1}{\SavedStyle\hspace{.2\ht0}%
  \hat{\vphantom{#1}}\kern\dimexpr2.2\LMpt+.7pt\relax\check{\vphantom{#1}}}{O}{c}{F}{T}{L}}%
}

\newcommand{\checkhat}[1]{\ThisStyle{\setbox0=\hbox{$\SavedStyle#1$}%
  \stackengine{0pt}{\SavedStyle#1}{\SavedStyle\hspace{.2\ht0}%
  \check{\vphantom{#1}}\kern\dimexpr2.2\LMpt+.7pt\relax\hat{\vphantom{#1}}}{O}{c}{F}{T}{L}}%
}

\newcommand{\hslashslash}{%
  \raisebox{.9ex}{%
    \scalebox{.7}{%
      \rotatebox[origin=c]{0}{$-$}%
    }%
  }%
}
\newcommand{\deltaslash}{%
  {%
   \vphantom{d}%
   \ooalign{\kern.05em\smash{\hslashslash}\hidewidth\cr$\delta$\cr}%
   \kern.05em
  }%
}

\usepackage{stmaryrd}

\newcommand{\sdf}{\phi} 
\newcommand{\sdfzero}{\phi_0} 
\newcommand{\force}{P} 
\newcommand{\youngs}{E} 
\newcommand{\poisson}{\nu} 
\newcommand{\estar}{{E^*}} 
\newcommand{\radius}{r} 
\newcommand{\point}{p} 
\newcommand{\normal}{q} 
\newcommand{\disp}{v} 
\newcommand{\indentation}{\delta} 
\newcommand{\curvature}{\kappa} 
\newcommand{\laplacian}{\Delta} 
\newcommand{\grad}{\nabla} 
\newcommand{\divg}{\nabla \cdot} 
\newcommand{\norm}[1]{\Vert #1 \Vert} 

\begin{document}

\maketitle
\thispagestyle{empty}
\pagestyle{empty}

\begin{abstract}
We introduce PROD (Palpative Reconstruction of Deformables), a novel method for reconstructing the shape and mechanical properties of deformable objects using elastostatic signed distance functions (SDFs). Unlike traditional approaches that rely on purely geometric or visual data, PROD integrates palpative interaction---measured through force-controlled surface probing---to estimate both the static and dynamic response of soft materials. We model the deformation of an object as an elastostatic process and derive a governing Poisson equation for estimating its SDF from a sparse set of pose and force measurements. By incorporating steady-state elastodynamic assumptions, we show that the undeformed SDF can be recovered from deformed observations with provable convergence. Our approach also enables the estimation of material stiffness by analyzing displacement responses to varying force inputs. We demonstrate the robustness of PROD in handling pose errors, non-normal force application, and curvature errors in simulated soft body interactions. These capabilities make PROD a powerful tool for reconstructing deformable objects in applications ranging from robotic manipulation to medical imaging and haptic feedback systems.
\end{abstract}


%
\IEEEpeerreviewmaketitle

\section{Introduction}

The ability to accurately reconstruct the geometry and mechanical properties of deformable objects is crucial for applications in robotics, medical imaging, and haptic interaction. Traditional approaches to reconstruction predominantly rely on visual scanning or contact-based surface tracing, which focus on external shape recovery but fail to capture the intrinsic material properties that govern an object's response to force. In contrast, biological systems leverage \emph{palpation}---a process that integrates touch, pressure, and force-induced deformations---to infer both geometric and mechanical characteristics of an object in real time.

Palpation has long been utilized in medical diagnostics to assess tissue stiffness and detect abnormalities. Techniques such as dynamic mechanical palpation enable the measurement of viscoelastic properties in soft tissues, providing quantitative insights into pathological conditions~\cite{Bouffandeau2025}. Force feedback as a result of varying tissue properties prove to be key in determining the efficacy of surgical incision techniques such as electrocautery \cite{El-Kebir2021d}, which bears significant interest in effectuating minimally invasive (micro)surgery \cite{El-Kebir2023c}. Additionally, studies on force modulation strategies have demonstrated the efficacy of active probing mechanisms in identifying local variations in stiffness within soft tissues~\cite{Palacio-Torralba2015, Konstantinova2017}. While these methods successfully extract material properties, they often lack integration with comprehensive shape reconstruction frameworks, leaving an opportunity for approaches that simultaneously recover geometry and elasticity.

Recent advances in 3D reconstruction have demonstrated the effectiveness of signed distance functions (SDFs) in capturing intricate geometries. SDF-based methods have been widely adopted for reconstructing rigid bodies due to their ability to encode surface boundaries implicitly while maintaining a continuous representation of object topology. For instance, Chen \emph{et al.}~\cite{Chen2023} introduced gSDF, a geometry-driven approach that reconstructs 3D hand-object interactions from monocular RGB images, while SDF-SRN~\cite{Lin2020} proposed a method for learning SDF representations from single-view images. Although these methods excel in modeling complex geometries, they do not account for the deformability of soft bodies, where shape varies in response to force.

Soft body reconstruction presents unique challenges due to the nonlinear and time-dependent nature of material deformations. While SDFs have primarily been applied to rigid objects, recent work has attempted to extend them to soft bodies. Xu \emph{et al.}~\cite{Xu2024} introduced GSurf, a framework that learns SDFs directly from Gaussian primitives to address issues such as incomplete reconstructions and fragmented surfaces, and Liu \emph{et al.}~\cite{Liu2022a} proposed ReDSDF, a regularized deep SDF approach that computes smooth distance fields at arbitrary scales to improve the representation of high-dimensional manifolds. Despite these advances, existing methods focus predominantly on shape recovery without explicitly modeling the underlying mechanical properties or force-driven deformations of the material. 


An alternative to model-driven approaches is the use of vision-based tactile sensors to infer shape and mechanical properties. GelSight~\cite{Yuan2015, Yuan2017} employs an elastomeric surface and high-resolution imaging to measure contact forces and surface deformations, enabling detailed reconstruction of object topography. Similarly, Insight~\cite{Sun2022} provides real-time 3D haptic sensing, capturing force directionality and surface compliance. While such external sensing systems offer rich tactile feedback, they rely on high-resolution imaging hardware and calibration procedures, limiting their adaptability in scenarios where direct interaction with soft objects is required. These systems also do not inherently model the dynamic response of a deformable body under continuous palpation, and do not strive to reconstruct the underformed shape of the body.

Intrinsic reconstruction methods that rely directly on palpative interactions offer an alternative to externally instrumented techniques. By actively probing an object and measuring its force response, it is possible to infer both its undeformed shape and material properties. However, prior work in this area has primarily focused on estimating local stiffness variations~\cite{Richey2023}, rather than reconstructing a full 3D model that captures both geometry and elastodynamic behavior. A key limitation of existing palpation-based methods is their reliance on empirical stiffness estimation, without integrating a formal mathematical model that accounts for the time-dependent deformation response.

To address these challenges, we introduce \emph{PROD} (Palpative Reconstruction of Deformables), a novel framework for reconstructing both the geometry and mechanical properties of deformable objects using elastodynamic signed distance functions. Unlike conventional SDF-based methods that assume rigid geometry, PROD models soft materials as elastodynamic systems governed by Newtonian mechanics and non-Hertzian contact mechanics. By applying controlled probing forces and measuring surface deformations, as illustrated in Figure~\ref{fig:flat punch}, our method reconstructs the signed distance field of the object while simultaneously estimating its material properties. In particular, the key contributions of our work are:
\begin{enumerate}
	\item \textbf{Elastostatic SDF Reconstruction}: We derive a governing Poisson equation that estimates the signed distance function from sparse pose data obtained through palpative interactions.
	\item \textbf{Recovery of Undeformed Geometry}: By leveraging steady-state elastodynamic assumptions as well as a combination of Hertzian and non-Hertzian contact theory, we develop an inversion method that recovers the undeformed SDF from deformed observations.
	\item \textbf{Estimation of Material Properties}: We propose a force-displacement analysis that enables the extraction of an object's Young's modulus and curvature, offering insights into its stiffness and compliance through a process of multiple probes.
	\item \textbf{Applications to Soft Robotics, Non-Destructive Testing, and Biomechanics}: We discuss potential uses of PROD in robotic grasping, medical robotics, non-destructive testing, and haptic feedback systems.
\end{enumerate}

By integrating force-based palpative exploration with computational elastostatics and contact mechanics, PROD provides a physics-informed approach to reconstructing soft objects that extends beyond purely geometric methods. In the following Sections, we formalize our approach, present theoretical guarantees for convergence, and validate our method on simulated deformable objects.

\section{Preliminaries}

We first consider the application of elastodynamic equations on signed distance functions, as defined in this Section.

\subsection{Signed Distance Functions}

\begin{figure}[t]
	\centering
	\includegraphics[width=0.5\linewidth]{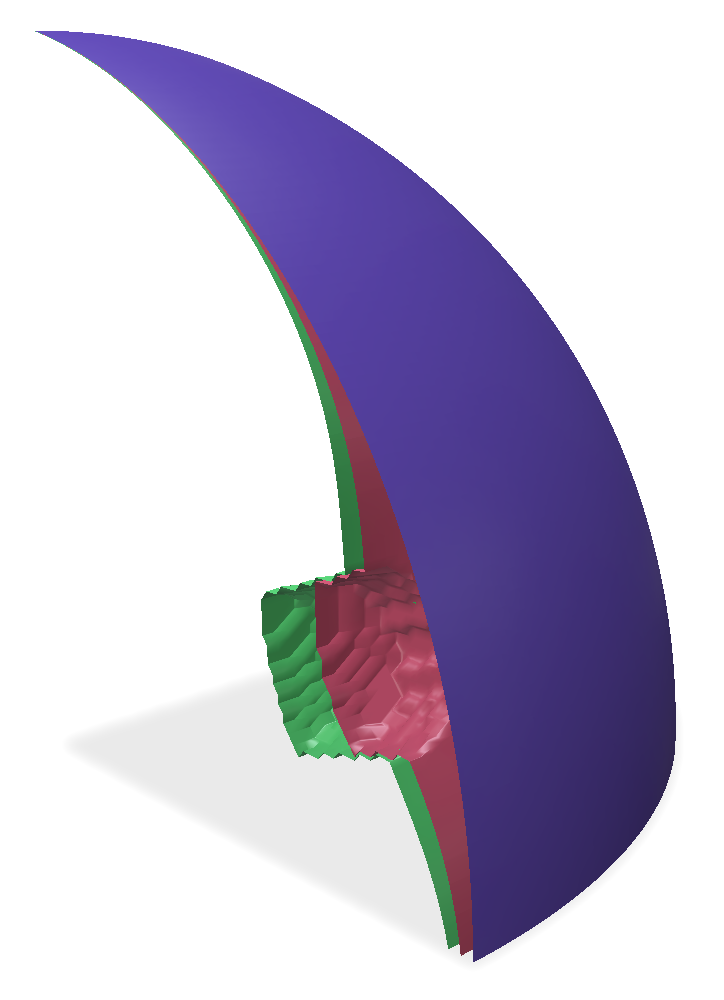}
	\caption{An illustration of the deformation of a spherical object under flat punch action. The undeformed surface is given in blue, whereas the red and green surface show deformation with increased pressure.}
	\label{fig:flat punch}
\end{figure}

Signed distance functions (SDFs) are defined as functions $\phi : \Omega \subseteq \mathbb{R}^n \to \mathbb{R}$ that satisfy a particular form of the \emph{eikonal equation} \cite[p.~93]{Evans2010}. The general eikonal equation is defined as \cite[p.~416]{Evans2010}
\begin{equation}\label{eq:general eikonal equation}
	\Vert \nabla \phi(x) \Vert = \frac{1}{\psi(x)},\quad x \in \Omega,
\end{equation}
where $\Omega$ is an open set in $\mathbb{R}^n$, and
$\psi : \mathbb{R}^n \to \mathbb{R}_+$ is a positively valued function. In physical terms, the solution $\phi(x)$ to this nonlinear partial differential equation can be interpreted as the shortest time needed to travel from the boundary $\partial\Omega$ to $x \in \Omega$, with $\psi(x)$ giving the
speed at $x$.

For a signed distance function, the speed will be unity throughout
the domain, and $x$ will not be constrained to lie in a subset of $\mathbb{R}^n$ such that $\phi : \mathbb{R}^n \to \mathbb{R}$ satisfying \eqref{eq:general eikonal equation} with $\psi$ $\Omega = \mathbb{R}^n$. Let $\Omega := \{\phi(x) \leq 0 : x \in \mathbb{R}^n \}$. Then, given that $\psi(x) \equiv 1$, the time to travel from the boundary $\partial\Omega$ to any $x \in \mathbb{R}^n$ will simply be equal to the shortest distance from $x$ to $\partial\Omega$. This distance is signed, meaning that for $x$ with $\phi(x) < 0$ will be in the interior of $\Omega$, $\phi(x) = 0$ implies that $x \in \partial\Omega$, and $\phi(x) > 0$ implies that $x \not\in \Omega$. We formally defined the signed distance function as follows:

\begin{definition}[Signed Distance Function (SDF)]\label{def:SDF}
	A function $\phi : \mathbb{R}^n \to \mathbb{R}$ is a \emph{signed distance function} of a set $\Omega$ if it satisfies:
	\begin{enumerate}[noitemsep]
	\item Eikonal equation:
	\begin{equation}\label{eq:eikonal equation}
		\Vert \nabla \phi(x) \Vert = 1,\quad x \in \mathbb{R}^n,
		\tag{EE}
	\end{equation}
	\item Zero distance at the boundary:
	\begin{equation}\label{eq:boundary zero}
		\phi(x) = 0,\quad \forall \ x \in \partial\Omega,
		\tag{BZ}
	\end{equation}
	\item Inward-facing normal at the boundary:
	\begin{equation}\label{eq:inward normal}
		\forall x \in \mathbb{R}^n \setminus \partial\Omega, \exists\epsilon > 0, \ \mathrm{s.t.} \ \phi\left( x + \epsilon \nabla\phi(x) \right) < \phi(x).
		\tag{IN}
	\end{equation}
\end{enumerate}
\end{definition}

\subsection{Elastodynamic Deformation}

Assuming a homogeneous and isotropic material with linear elasticity and a known Young's modulus $E$ and a density $\rho$, the elastodynamic equation under a force field $f : \Omega_0 \times \mathbb{R}_+ \to \mathbb{R}_+$ reads
\begin{equation}\label{eq:elastodynamic equation}
	\rho \frac{\partial^2 v}{\partial t^2} = \nabla \cdot \sigma + f(x, t), \quad \sigma = E \varepsilon,
\end{equation}
where $v(x, t)$ is the time varying \emph{deformation field} and $\varepsilon$ is the strain tensor. We proceed by developing a similar elastodynamic equation that applies directly to a signed distance field.

\section{Reconstruction of Elastodynamic SDFs}

In light of the elastodynamic deformation equation \eqref{eq:elastodynamic equation}, as well as the Definition~\ref{def:SDF}, we directly obtain the following elastodynamic equation that governs the signed distance function:
\begin{equation}
	\rho \frac{\partial^2 \phi}{\partial t^2} + \gamma \frac{\partial \phi}{\partial t} - E \Delta\phi = u(t) w_r (x - p(t)) \langle q(t), \nabla \phi \rangle,
\end{equation}
where $\phi = \phi_0 + \delta\phi$, with $\phi_0$ denoting the undeformed state. Let $w_r$ be a weighting function parameterized by a radius $r$, such that $\int_{\Omega_0} w_r(x) dx = \pi r^2$. This is commonly known as a \emph{flat punch} model, wherein a cylinder is pressed with its flat side into a surface in the normal direction. This is illustrated in Figure~\ref{fig:flat punch}.

In a steady-state case where $(u, p, q)$ are constant, we obtain the following elastostatic equation after expanding $\phi$:
\begin{equation}
\begin{split}
	E\Delta (\delta\phi) &= -u w_r (x - p) \langle q, \nabla \phi_0 + \nabla\delta\phi \rangle - E \Delta \phi_0, \\
	\frac{\partial \delta\phi}{\partial n} &= -1 \ \mathrm{on} \ \partial\Omega_0.\\
\end{split}
\end{equation}

We now consider whether it is prudent to assume steady-state conditions in a biomedically relevant scenario.

\subsection{Correctness of Steady-State Assumptions}

We proceed by approximating the settling time of a linearized elastodynamic system under small perturbations. We obtain the following PDE:
\begin{equation*}
	\rho \frac{\partial^2 \delta\phi}{\partial t^2} + \gamma \frac{\partial \delta\phi}{\partial t} - E \Delta \delta\phi = 0.
\end{equation*}
Performing a modal decomposition in the Laplace domain, with solution $\delta\phi(x, t) = \hat{\phi}(x) \exp(s t)$, we obtain:
\begin{equation*}
	\rho s^2 \hat{\phi} + \gamma s \hat{\phi} - E \Delta \hat{\phi} = 0.
\end{equation*}
Dividing by $\hat{\phi}$, and taking $\lambda$ to be the eigenvalues of the Laplacian $-\Delta$, we have
\begin{equation*}
	\rho s^2 + \gamma s + E\lambda = 0.
\end{equation*}
Given this dispersion relation, we find the (elastic) disturbance propagation speed to be $c \approx \sqrt{E/\rho}$ assuming that $\gamma$ is small relative to $E$ and $\rho$. Let $\ell$ be a characteristic length scale for the deformable object under consideration, for instance its thickness. We define $T_e := \ell/c$ to be the time it takes for an elastic wave to propagate a distance the length of the characteristic length. If the contact time, or forcing time, $T_{c}$ is such that $T_c \gg T_e$, then we may safely ignore dynamic terms and assume attainment of a steady-state condition within a period $T_c$.


We verify this time scale separation assumptions on human adipose tissue, which has a density of $\rho = 960$ kg/m\textsuperscript{3} and a Young's modulus $E$ between 4.48--11.50 kPa \cite{Wenderott2020}. Assuming that the contact time is at least $T_c = 100$ ms and we have a characteristic length scale of $\ell = 5$ mm, we find $T_e \geq 2.31$ ms is such that $T_c \gg T_e$ is satisfied.


\section{Unique SDF Estimation from Pose Data}\label{sec:SDF estimation}

Let $\Omega \subseteq \mathbb{R}^n$ be an unknown \emph{continuum}, or compact, closed, and path connected set. We assume that we are given $N$ \emph{poses} $Z_N := \{z_i\}_{i=1}^N$, where $z_i = (p_i, q_i) \in \mathbb{R}^6$ such that each $p_i \in \mathbb{R}^3$ is a position on $\partial \Omega$ and $q_i \in \mathcal{B}_1^3$ is the inward surface normal of $\Omega$ at $p_i$. We seek to reconstruct the unknown signed distance function $\phi : \mathbb{R}^n \to \mathbb{R}$ of $\Omega$ on the basis of these $N$ poses.

\begin{theorem}\label{thm:pseudo-SDF}
	Let $\Omega \subseteq \mathbb{R}^n$ be an unknown \emph{continuum}, or compact, closed, and path connected set. Assume that $N$ \emph{poses} $Z_N := \{z_i\}_{i=1}^N$ are given, where $z_i = (p_i, q_i) \in \mathbb{R}^6$ such that each $p_i \in \mathbb{R}^3$ is a position on $\partial \Omega$ and $q_i \in \mathcal{B}_1^3$ is the inward surface normal of $\Omega$ at $p_i$.
	
	Let $\hat{q} : \mathbb{R}^n \to \mathcal{B}_1^n$ be the interpolated normal field such that $\hat{q}(p_i) = q_i$. Then, the solution to the Poisson equation
	\begin{equation}\label{eq:Poisson reconstruction of SDF}
		\Delta \phi = \nabla \cdot \hat{q}.
	\end{equation}
	yields the a \emph{pseudo-signed distance function} of $\partial\Omega$ satisfying \eqref{eq:boundary zero} and \eqref{eq:inward normal}.
\end{theorem}

\begin{proof}
Let $K := \mathrm{convex}(\{p_i\}_{i=1}^N)$ be the convex hull of the given positions, and let $\phi_0 (x) := \Vert x - \mathrm{center}(K) \Vert - \mathrm{diam}(K)/2$ be a circumscribing sphere that encloses $K$. This will be our initial guess for the SDF $\phi$.

Let $\hat{q} : K \to \mathcal{B}_1^n$ be the interpolated normal field on $K$ such that $\hat{q}(p_i) = q_i$. Our problem reduces to an energy minimization problem of the following objective functional:
\begin{equation}
	F_0 (\phi) := \int_K \Vert \nabla\phi(x) - \hat{q}(x) \Vert^2 \Vert dx.
\end{equation}

Expanding the functional $F_0$, we obtain:
\begin{equation}
	F_0 (\phi) = \int_K \left( \Vert \nabla \phi(x) \Vert^2 - 2 \nabla \phi (x) \cdot \hat{q}(x) + \Vert \hat{q}(x) \Vert^2 \right) dx.
\end{equation}
The term $\Vert \hat{q}(x) \Vert^2$ can be dropped since it does not depend on $\phi$. Furthermore, we can shrink $K$ to the unknown $\Omega$ without loss of generality. 
Hence, let the new energy functional be $F (\phi) := \int_\Omega \left( \Vert \nabla \phi (x) \Vert^2 - 2 \nabla \phi (x) \cdot \hat{q}(x) \right) dx$.

Computing the first variation of $F$, we obtain
\begin{equation*}
\begin{split}
	\delta F &= \left.\frac{d}{d\epsilon} F(\phi + \epsilon \delta \phi) \right|_{\epsilon = 0} \\
	&= \int_K \frac{d}{d\epsilon} \left( \Vert \nabla (\phi + \epsilon \delta\phi) \Vert^2 - 2 \nabla (\phi + \epsilon \delta\phi) \cdot \hat{q} \right) dx. \\
	&= \int_K \left( 2 \nabla \phi \cdot \nabla(\delta \phi) - 2 \hat{q} \cdot \nabla(\delta\phi) \right) dx.
\end{split}
\end{equation*}

Next, we apply integration by parts to obtain
\begin{equation*}
\begin{split}
	\int_\Omega \nabla \phi \cdot \nabla(\delta \phi) dV &= \oint_{\partial\Omega} \phi \ \delta\phi \hat{n} dS - \int_{\Omega} \Delta\phi \ \delta\phi dV \\
	&= - \int_{\Omega} \Delta\phi \ \delta\phi dV,
\end{split}
\end{equation*}
since $\phi$ is zero along $\partial\Omega$. Similarly, we have
\begin{equation*}
\begin{split}
	\int_\Omega \hat{q} \cdot \nabla(\delta\phi) dV &= \oint_{\partial\Omega} \hat{q} \cdot \delta\phi \cdot \hat{n} dS - \int_{\Omega} (\nabla \cdot \hat{q} ) \delta\phi dV \\
	&= - \int_{\Omega} (\nabla \cdot \hat{q} ) \delta\phi dV,
\end{split}
\end{equation*}
since $\delta\phi$ is zero along the surface $\partial\Omega$. We can now construct the following Euler-Lagrange equation by setting $\delta F = 0$ for all variations $\delta\phi$ gives $-2\Delta\phi + 2\nabla\cdot\hat{q} = 0$, yielding   Poisson equation \eqref{eq:Poisson reconstruction of SDF}.
\end{proof}

\begin{remark}
	The Poisson equation \eqref{eq:Poisson reconstruction of SDF} arises directly when enforcing that $\nabla \phi$ aligns with normal field $\hat{q}$. The Poisson equation ensures that the Laplacian of the SDF matches the divergence of the normal field. If $\hat{q}$ is a valid normal field of some surface, the solution of this \eqref{eq:Poisson reconstruction of SDF} will recover an SDF up to some additive constant.
\end{remark}

To resolve the additive constant, we enforce Dirichlet boundary conditions $\phi(p_i) = 0$ for $i=1,\ldots,N$, yielding the following PDE to uniquely reconstruct $\phi$ from $Z_N$, which we denote by $\phi_N$:
\begin{equation}\label{eq:Poisson PDE for unique SDF}
	\Delta \phi_N = \nabla \cdot \hat{q}, \quad \phi_N(p_i) = 0, \ i=1,\ldots,N.
\end{equation}

Let us now verify if the eikonal equation \eqref{eq:general eikonal equation} holds for $\psi \equiv 0$. Since the solution to Poisson equation \eqref{eq:Poisson PDE for unique SDF} minimizes the energy functional $E(\phi) = \int_{\Omega} \Vert \nabla \phi - \hat{q} \Vert^2 dx$, we ensure that $\nabla \phi$ \emph{aligns} with $\hat{q}$. Since $\hat{q}$ is normalized, we have $\Vert \hat{q}(x) \Vert = 1$ for all $x \in \mathbb{R}^n$. Hence $\nabla\phi \approx \hat{q}$ implies that $\Vert \nabla \phi \Vert \approx 1$.

Since \emph{gradient alignment}, and not \emph{gradient magnitude}, is enforced in the Poisson equation, we note that the correct gradient magnitude is obtained at the zero-level set $\{ x \in \mathbb{R}^n : \phi(x) = 0 \}$ given the Dirichlet boundary conditions $\phi(p_i) = 0$. Additionally, this boundary condition is essential for uniqueness, since in its absence $\phi + c$ is also a correct solution for any constant $c$. We proceed by showing how we can turn the pseudo-SDF obtained in Theorem~\ref{thm:pseudo-SDF} can be transformed into an equivalent true SDF that satisfies \eqref{eq:eikonal equation}.

\subsection{Global Satisfaction of the Eikonal Condition}

We introduce the following iterative evolution step to enforce the eikonal equation \eqref{eq:eikonal equation} across the entire domain:
\begin{equation}\label{eq:reinitialization PDE}
    \partial_t \phi + \operatorname{sign}(\phi_0)(\Vert\nabla \phi\Vert - 1) = 0.
\end{equation}
When initialized using a pseudo-SDF $\phi_0$, this PDE will produce a signed distance function that satisfies \eqref{eq:eikonal equation}. At each iteration, the equation is a Hamilton--Jacobi PDE, which can be efficiently solved using upwind schemes.

Since \eqref{eq:reinitialization PDE} is a time-evolving PDE, we propose an alternative iterative update scheme that preserves the zero-level set:

\begin{theorem}
    Let $\hat{\phi}$ be a pseudo-signed distance function such that $\{ x \in \mathbb{R}^n : \phi(x) = 0 \} = \partial\Omega$ of some volume $\Omega$. Suppose that \eqref{eq:eikonal equation} is not satisfied for all $x \in \mathbb{R}^n$. Then, Algorithm~\ref{algo:SDF reinitialization} produces a true signed distance function that satisfies Definition~\ref{def:SDF} within an error $\epsilon > 0$ and preserves the zero-level set of $\hat{\phi}$:
    \begin{algorithm}
    \KwData{$\epsilon > 0$, $\Delta t > 0$, $\hat{\phi}$}
    \KwResult{$\bar{\phi}$}
    $k \gets 0$\;
    $\phi^{(0)} \gets \hat{\phi}$\;
    \While{$\max_{x} |\Vert \nabla \phi^{(k)}(x) \Vert - 1| \geq \epsilon$}{
        Solve $\phi^{(k+1)} \gets \phi^{(k)} - \Delta t \cdot \operatorname{sign}(\hat{\phi}) (\Vert\nabla \phi^{(k)}\Vert - 1)$\;
        $k \gets k+1$\;
    }
    $\bar{\phi} \gets \phi^{(k)}$\;
    \caption{Reinitialization of a pseudo-SDF into a zero-level set preserving true SDF.}
    \label{algo:SDF reinitialization}
    \end{algorithm}
\end{theorem}

\begin{proof}
    We show that the evolution equation
    \begin{equation*}
        \partial_t \phi + \operatorname{sign}(\phi_0)(\Vert\nabla \phi\Vert - 1) = 0
    \end{equation*}
    ensures that:
    \begin{equation*}
		\phi^{(0)}(x) = \hat{\phi}(x) = 0 \ \forall x \in \partial\Omega \ \Rightarrow \ \phi^{(k)}(x) = 0 \ \forall x \in \partial\Omega, \ \forall k.
	\end{equation*}
    The proof follows by stationarity of the zero-level set under the evolution PDE. Since \eqref{eq:reinitialization PDE} only modifies $\phi$ based on $\Vert\nabla \phi\Vert - 1$, it does not change the values of $\phi$ at $\phi = 0$, because at the zero-level set, $\operatorname{sign}(\phi(x)) = 0$ and remains unchanged throughout the process.

    More formally, for any fixed level set $\phi = c$, the reinitialization PDE can be rewritten in level-set advection form:
    \begin{equation*}
        \partial_t \phi + \mathbf{v} \cdot \nabla \phi = 0, \quad \text{where } \mathbf{v} = -\operatorname{sign}(\phi) (\Vert\nabla \phi\Vert - 1) \frac{\nabla \phi}{\Vert\nabla \phi\Vert}.
    \end{equation*}
    Since the velocity $\mathbf{v}$ is always directed \emph{along the normal to the level set}, and is zero whenever $\|\nabla \phi\| = 1$, the zero-level set itself remains stationary during evolution. This ensures that:
    \begin{equation*}
        \{x : \phi^{(0)}(x) = 0\} = \{x : \phi^{(k)}(x) = 0\}, \quad \forall k.
    \end{equation*}
    Thus, Algorithm~\ref{algo:SDF reinitialization} correctly reinitializes $\hat{\phi}$ into a true signed distance function while preserving its zero-level set. This completes the proof.
\end{proof}

\section{Convergence of SDF Estimate}

We would now like to show that the zero level set of $\phi_N$, constructed from $Z_N$ using the approach from Section~\ref{sec:SDF estimation}, converges to that of $\phi$ for $N \to \infty$. We do so by considering the Hausdorff distances between these two level sets:
\begin{equation*}
\begin{split}
	&d_N := d_{\mathrm{H}} (\phi_N\inv(0), \phi\inv(0)) = d_{\mathrm{H}} (\partial\Omega_N, \partial\Omega) \\
	&= d_{\mathrm{H}} (\{x : \phi_N (x) = 0 \}, \{x : \phi (x) = 0 \}) \\
	&= \max\left\{ \sup_{x \in \phi_N\inv (0)} \inf_{y \in \phi\inv(0)} \Vert x - y \Vert, \sup_{x \in \phi\inv (0)} \inf_{y \in \phi_N\inv(0)} \Vert x - y \Vert \right\}.
\end{split}
\end{equation*}
We would like to show that $\lim_{N \to \infty} d_N = 0$.

Before we continue, we make the following lax assumption about $Z_N$:
\begin{assumption}\label{assump:convergence of SDF estimate}
	We assume that $Z_N$ is self-consistent for all $N$, i.e., all samples $z_i = (p_i, q_i) \in Z_N$ lie on a continuous manifold $\partial\Omega$ with Lipschitz constant $L$. Furthermore, we assume that for $N \to \infty$, the set $\{p_i\}_{i=1}^N$ becomes dense on $\partial\Omega$ and that the interpolated normal field $\hat{q}_N$ converges uniformly to the true normal field $q$. Finally, we assume that $q_N, q \in (L^2(\Omega_0))^n$ and $\partial\Omega \in C^{1,1}(\Omega_0)$.
\end{assumption}

\begin{proposition}
	Let $Z_N$ be a set of $N$ poses on a Lipschitz continuous manifold $\partial\Omega$ that satisfies Assumption~\ref{assump:convergence of SDF estimate}, and let $\phi_N$ be the signed distance function constructed from $Z_N$ based by solving \eqref{eq:Poisson PDE for unique SDF} and regularizing it through Algorithm~\ref{algo:SDF reinitialization}. Let $\phi$ be the true SDF that defines $\partial\Omega$.
	
	Then, for $N \to \infty$, the following statements hold
	\begin{enumerate}
		\item $\lim_{N \to \infty} \Vert \phi_N - \phi \Vert_{H^2(\Omega_0)} = 0$,
		\item $\lim_{N \to \infty} d_N = \lim_{N \to \infty} d_{\mathrm{H}} (\phi_N\inv(0), \phi\inv(0)) = 0$,
		\item $d_{\mathrm{H}} (\Omega_N, \Omega) = d_{\mathrm{H}} (\{x : \phi_N (x) \leq 0 \}, \{x : \phi (x) \leq 0 \}) = 0$.
	\end{enumerate}
\end{proposition}

\begin{proof}
Since \eqref{eq:Poisson PDE for unique SDF} is a second-order elliptic PDE, we find that $\phi_N, \phi \in H^2(\Omega_0) \cap H_0^1 (\Omega_0)$, where $H_0^1 (\Omega_0)$ is a subpsace of the first-order Sobolev space $H^1 (\Omega_0) = \{ u \in L^2(\Omega_0) : \nabla u \in (L^2 (\Omega_0))^n \}$. By elliptic regularity theory, in particular the Babu\v{s}ka--Lax--Milgram theorem \cite[\S6.2.1, p.~315]{Evans2010}, we find existence and uniqueness of \eqref{eq:Poisson PDE for unique SDF} are satisfied, as well as the fact that the solution depends continuously on $q$ and $q_N$. Furthermore, by Assumption~\ref{assump:convergence of SDF estimate}, given that $\phi_N, \phi \in H^2(\Omega_0)$, it follows that this continuous dependence extends to the gradient of $q$ and $q_N$, yielding:
\begin{equation*}
	\Vert \phi_N - \phi \Vert_{H^2 (\Omega_0)} \leq C \Vert q_N - q \Vert_{L^2(\Omega_0)},
\end{equation*}
for some constant $C$. Hence, given uniform convergence of $\lim_{N \to \infty} \Vert q_N - q \Vert_\infty = 0$, we find that $\lim_{N \to \infty} \Vert \phi_N - \phi \Vert_{H^2 (\Omega_0)} = 0$. Finally, given the Lipschitz continuity of $\phi$, we can show that the zero level sets $\phi_N\inv (0)$ converges to $\phi\inv (0)$ in the Hausdorff sense using an $\epsilon$--$\delta$ argument.

For any $\epsilon > 0$, there exists $\delta > 0$ such that
\begin{equation*}
	|\phi(x)| < \delta/2 \ \Rightarrow \ d(x, \partial\Omega) < \epsilon/2.
\end{equation*}

Furthermore, there exists $N$ such that $|\phi(x) - \phi_N(x)| < \delta/2$ by uniform convergence, yielding
\begin{equation*}
	|\phi_N(x)| \leq |\phi(x) - \phi_N(x)| + |\phi(x)| < \delta \ \Rightarrow \ d(x, \partial\Omega_N) < \epsilon/2.
\end{equation*}

Finally, we have that
\begin{equation*}
	|\phi(x)|, |\phi_N (x)|, |\phi(x) - \phi_N (x)| < \delta/2
\end{equation*}
implies
\begin{equation*}
	d_{\mathrm{H}} (\partial\Omega, \partial\Omega_N) \leq d(x, \partial\Omega) + d(x, \partial\Omega_N) < \epsilon,
\end{equation*}
which follows by the triangle inequality. Hence, we have shown that $d_{\mathrm{H}} (\phi_N\inv(0), \phi\inv(0))$ converges to zero for $N \to \infty$. If $\phi_N$ and $\phi$ both satisfy the eikonal equation \eqref{eq:eikonal equation}, then it follows that $\lim_{N \to \infty} d_{\mathrm{H}} (\Omega, \Omega_N) = 0$. This completes the proof.
\end{proof}

We now consider the problem of uniquely determining undeformed SDFs from deformed pose data.

\section{Unique Undeformed SDF Estimation from Deformed Pose Data}

We consider the following assumptions:

\begin{assumption}[Material Properties]\label{assump:material properties}
We assume the object is composed of a homogeneous and isotropic linear elastic material with Young's modulus $\youngs$ and Poisson's ratio $\poisson$. We define the plane strain modulus as $\estar = \youngs / (1 - \poisson^2)$. All deformations are assumed to be purely elastic. In its unloaded state, the object's shape is described by a signed distance function (SDF) $\sdfzero$.
\end{assumption}

\begin{assumption}[Force Application and Contact]\label{assump:force application}
A total normal force $\force_i$ is applied to the object at a point $\point_i$ on its surface, directed along the surface normal $\normal_i$. The force is distributed over a circular area of radius $\radius$, modelled as contact with a rigid, flat-ended cylindrical punch of radius $\radius$. We assume no tangential forces (frictionless contact). After deformation, the object's shape is described by the deformed SDF $\sdf_i$. We assume the contact results in the application point being on the deformed surface, i.e., $\sdf_i(\point_i) = 0$. Furthermore, we assume that the material point corresponding to $\point_i$ was initially inside the undeformed object, i.e., $\sdfzero(\point_i) < 0$.
\end{assumption}


Let us now treat the case of zero curvature $\kappa$.

\begin{proposition}[Undeformed SDF Estimate: Flat Case]\label{prop:flat-case reconstruction}
Consider the loading conditions of Assumption~\ref{assump:force application} applied to an initially \textit{flat} elastic half-space ($\curvature = 0$). Let $\point_i$ be the center of the contact area after force $\force_i$ is applied. The value of the undeformed SDF $\sdfzero$ evaluated at $\point_i$ is approximately related to the flat-case indentation depth $\indentation_{\mathrm{flat}, i}$:
\begin{equation*}
    \sdfzero(\point_i) \approx -\indentation_{\mathrm{flat}, i} = -\frac{\force_i (1 - \poisson^2)}{2 \youngs \radius} = -\frac{\force_i}{2 \estar \radius}.
\end{equation*}
\end{proposition}

\begin{proof}
Let $\point_i$ be the location on the deformed surface where the load is centered ($\sdf_i(\point_i) = 0$). Let $\mathbf{x}_{0,i}$ be the material point in the undeformed configuration that moves to $\point_i$ after deformation, such that $\point_i = \mathbf{x}_{0,i} + \disp_i(\mathbf{x}_{0,i})$, where $\disp_i$ is the displacement field.

Under the assumption of a rigid flat punch of radius $\radius$ pressing into an elastic half-space, the material surface within the contact area undergoes a uniform normal displacement $\indentation_{\mathrm{flat}, i}$. The relationship between the total force $\force_i$ and this displacement is given by standard contact mechanics results \cite{Johnson2003a} as:
\begin{equation*}
    \indentation_{\mathrm{flat}, i} = \frac{\force_i (1 - \poisson^2)}{2 \youngs \radius} = \frac{\force_i}{2 \estar \radius}.
\end{equation*}
The displacement vector at the center point $\mathbf{x}_{0,i}$ is predominantly normal to the surface. If $\normal_i$ is the outward normal at $\point_i$, the displacement vector is approximately $\disp_i(\mathbf{x}_{0,i}) \approx -\indentation_{\mathrm{flat}, i} \normal_i$.

We relate the deformed and undeformed SDFs using a first-order Taylor expansion of $\sdfzero$ around $\mathbf{x}_{0,i}$:
\begin{equation*}
    \sdfzero(\point_i) = \sdfzero(\mathbf{x}_{0,i} + \disp_i(\mathbf{x}_{0,i})) \approx \sdfzero(\mathbf{x}_{0,i}) + \grad \sdfzero(\mathbf{x}_{0,i}) \cdot \disp_i(\mathbf{x}_{0,i}).
\end{equation*}
Assuming the undeformed point $\mathbf{x}_{0,i}$ was on or very close to the original surface, $\sdfzero(\mathbf{x}_{0,i}) \approx 0$. The gradient $\grad \sdfzero(\mathbf{x}_{0,i})$ is the normal vector $\normal_{0,i}$ to the undeformed surface at $\mathbf{x}_{0,i}$. For small deformations, we approximate $\normal_{0,i} \approx \normal_i$. Substituting the displacement vector:
\begin{equation*}
    \sdfzero(\point_i) \approx 0 + \normal_i \cdot (-\indentation_{\mathrm{flat}, i} \normal_i) = -\indentation_{\mathrm{flat}, i} (\normal_i \cdot \normal_i).
\end{equation*}
Since $\normal_i$ is a unit vector, $\normal_i \cdot \normal_i = 1$. Therefore,
\begin{equation*}
    \sdfzero(\point_i) \approx -\indentation_{\mathrm{flat}, i} = -\frac{\force_i}{2 \estar \radius}.
\end{equation*}
This shows that the value of the undeformed SDF at the final contact point is approximately the negative of the indentation depth required in the flat case.
\end{proof}


We now incorporate the effect of initial surface curvature $\curvature$ using the additive approximation developed previously. This result is related to the superposition of a Boussinesq's solution to a point load \cite{Itou2020}.

\begin{proposition}[Undeformed SDF Estimate: Curved Case]\label{prop:curved-case reconstruction}
Consider the loading conditions of Assumption~\ref{assump:force application} applied to an initially curved elastic body. Let $\point_i$ be the center of the contact area after force $\force_i$ is applied. Let the initial undeformed shape near $\point_i$ have a local curvature represented by $\curvature(\mathbf{x})$. The value of the undeformed SDF $\sdfzero$ evaluated at $\point_i$ is approximately related to the total indentation depth $\indentation_{\mathrm{total}, i}$, which includes both elastic deformation and geometric flattening:
\begin{equation*}
    \sdfzero(\point_i) \approx -\indentation_{\mathrm{total}, i} \approx -\left( \indentation_{\mathrm{flat}, i} + \indentation_{\mathrm{geom}, i} \right),
\end{equation*}
where $\indentation_{\mathrm{flat}, i} = \force_i / (2 \estar \radius)$ is the flat-case elastic indentation, and $\indentation_{\mathrm{geom}, i}$ represents the average initial geometric gap over the contact area due to curvature.

Specifically:
\begin{itemize}
    \item If the curvature is approximately constant $\curvature_i$ over the contact area:
    \begin{equation*}
        \indentation_{\mathrm{geom}, i} \approx \frac{\curvature_i \radius^2}{4}, \quad \text{so} \quad \sdfzero(\point_i) \approx -\left( \frac{\force_i}{2 \estar \radius} + \frac{\curvature_i \radius^2}{4} \right).
    \end{equation*}
    \item If the curvature $\curvature(\rho, \theta)$ varies significantly over the contact area (in polar coordinates $(\rho, \theta)$ centered at $\point_i$ on the tangent plane), and approximating the initial gap as $h(\rho, \theta) \approx \curvature(\rho, \theta) \rho^2 / 2$:
    \begin{equation*}
        \indentation_{\mathrm{geom}, i} = \langle h \rangle_i = \frac{1}{\pi \radius^2} \int_{0}^{2\pi} \int_{0}^{\radius} \frac{\curvature(\rho, \theta) \rho^2}{2} \rho \, d\rho \, d\theta.
    \end{equation*}
     Then:
    \begin{equation*}
        \sdfzero(\point_i) \approx -\left( \frac{\force_i}{2 \estar \radius} + \langle h \rangle_i \right).
    \end{equation*}
\end{itemize}
\end{proposition}

\begin{proofsketch}
The total indentation $\indentation_{\mathrm{total}, i}$ required at the center $\point_i$ must achieve two goals: (1) deform the material elastically as if it were flat, corresponding to $\indentation_{\mathrm{flat}, i}$, and (2) close the initial geometric gap $h(\mathbf{x})$ that exists between the tangent plane at $\point_i$ and the curved surface over the contact radius $\radius$.

The additive approximation assumes these contributions can be superimposed:
\begin{equation*}
    \indentation_{\mathrm{total}, i} \approx \indentation_{\mathrm{flat}, i} + \indentation_{\mathrm{geom}, i}.
\end{equation*}
The term $\indentation_{\mathrm{geom}, i}$ represents the effective depth needed to flatten the initial curvature over the contact area. A reasonable estimate for this is the average value of the initial gap $h(\mathbf{x})$ over the circular contact area $A = \pi \radius^2$:
\begin{equation*}
    \indentation_{\mathrm{geom}, i} = \langle h \rangle_i = \frac{1}{A} \iint_A h(\mathbf{x}) \, dA.
\end{equation*}
For small distances $\rho$ from the center $\point_i$ on the tangent plane, the gap due to curvature $\curvature$ is approximately $h(\rho) \approx \curvature \rho^2 / 2$.

\textit{Case 1: Constant Curvature $\curvature_i$.}
The average gap is:
\begin{equation*}
\begin{split}
	\langle h \rangle_i &= \frac{1}{\pi \radius^2} \int_0^{2\pi} \int_0^{\radius} \left( \frac{\curvature_i \rho^2}{2} \right) \rho \, d\rho \, d\theta = \frac{2\pi}{\pi \radius^2} \frac{\curvature_i}{2} \int_0^{\radius} \rho^3 \, d\rho \\
	&= \frac{\curvature_i}{\radius^2} \left[ \frac{\rho^4}{4} \right]_0^{\radius} = \frac{\curvature_i \radius^2}{4}.
\end{split}
\end{equation*}

\textit{Case 2: Varying Curvature $\curvature(\rho, \theta)$.}
Assuming $h(\rho, \theta) \approx \curvature(\rho, \theta) \rho^2 / 2$, the average gap becomes the integral expression stated in the Proposition.

Using the same Taylor expansion argument as in the proof of Proposition~\ref{prop:flat-case reconstruction}, but recognizing that the total required normal displacement is now $\indentation_{\mathrm{total}, i}$, we arrive at:
\begin{equation*}
    \sdfzero(\point_i) \approx -\indentation_{\mathrm{total}, i} \approx -(\indentation_{\mathrm{flat}, i} + \indentation_{\mathrm{geom}, i}).
\end{equation*}
Substituting the expressions for $\indentation_{\mathrm{flat}, i}$ and $\indentation_{\mathrm{geom}, i}$ gives the results stated.
\end{proofsketch}

\section{Young's Modulus Determination}

We can estimate the material's Young's modulus by applying incremental forces at approximately the same location and measuring the resulting change in indentation depth. This method relies on the insight that the geometric contribution to indentation, due to initial curvature, cancels out when considering the difference between two indentation states.

\begin{proposition}[Young's Modulus Estimation]\label{prop:youngs_modulus}
Let the conditions of Assumption~\ref{assump:force application} hold for two consecutive probes, $i$ and $i+1$, at approximately the same location. Let the applied forces be $\force_i$ and $\force_{i+1}$, with $\force_{i+1} > \force_i$. Let the corresponding contact points be $\point_i$ and $\point_{i+1}$, and assume the surface normal $\normal$ remains approximately constant ($\normal_i \approx \normal_{i+1} \approx \normal$). Assume the Poisson's ratio $\poisson$ is known.

If the displacement of the probe between the two steps, $\Delta \point = \point_{i+1} - \point_i$, occurs purely along the normal direction $\normal$, then the magnitude of this displacement corresponds to the change in indentation depth, $\Delta\indentation = \norm{\Delta \point}$. The plane strain modulus $\estar$ can be estimated as:
\begin{equation*}
    \estar \approx \frac{\force_{i+1} - \force_i}{2 \radius \Delta\indentation} = \frac{\force_{i+1} - \force_i}{2 \radius \norm{\point_{i+1} - \point_i}}.
\end{equation*}
Consequently, the Young's modulus $\youngs$ is estimated as:
\begin{equation*}
    \youngs = \estar (1 - \poisson^2) \approx \frac{(\force_{i+1} - \force_i)(1 - \poisson^2)}{2 \radius \norm{\point_{i+1} - \point_i}}.
\end{equation*}
\end{proposition}

\begin{proof}
We use the additive approximation for the total indentation depth from Proposition~\ref{prop:curved-case reconstruction}. Let $\indentation_j$ be the total indentation depth under force $\force_j$ at point $\point_j$. Assuming the curvature $\curvature$ and punch radius $\radius$ are constant for both small incremental steps at approximately the same location:
\begin{equation*}
    \indentation_i \approx \frac{\force_i}{2 \estar \radius} + \indentation_{\mathrm{geom}}, \quad \indentation_{i+1} \approx \frac{\force_{i+1}}{2 \estar \radius} + \indentation_{\mathrm{geom}},
\end{equation*}
where $\indentation_{\mathrm{geom}}$ is the geometric indentation term (e.g., $\curvature \radius^2 / 4$ for constant curvature).

The change in indentation depth, $\Delta\indentation$, between state $i$ and $i+1$ is:
\begin{equation*}
    \Delta\indentation = \indentation_{i+1} - \indentation_i \approx = \frac{\force_{i+1} - \force_i}{2 \estar \radius}.
\end{equation*}
Note that the geometric term $\indentation_{\mathrm{geom}}$ cancels out.

Under the crucial assumption that the probe displacement $\Delta \point = \point_{i+1} - \point_i$ occurs purely along the normal direction $\normal$ (i.e., $\Delta \point$ is parallel to $\normal$), the magnitude of this displacement is precisely the change in indentation depth $\Delta\indentation = \norm{\point_{i+1} - \point_i}$.

More generally, if the motion is not purely normal, $\Delta\indentation = |\langle \point_{i+1} - \point_i, \normal \rangle|$. Assuming pure normal motion simplifies to $\Delta\indentation = \norm{\point_{i+1} - \point_i}$.

Substituting this into the expression for $\Delta\indentation$:
\begin{equation*}
    \norm{\point_{i+1} - \point_i} \approx \frac{\force_{i+1} - \force_i}{2 \estar \radius}.
\end{equation*}
Rearranging to solve for $\estar$:
\begin{equation*}
    \estar \approx \frac{\force_{i+1} - \force_i}{2 \radius \norm{\point_{i+1} - \point_i}}.
\end{equation*}
Finally, using the definition $\estar = \youngs / (1 - \poisson^2)$, we find the estimate for Young's modulus:
\begin{equation*}
    \youngs \approx \frac{(\force_{i+1} - \force_i)(1 - \poisson^2)}{2 \radius \norm{\point_{i+1} - \point_i}}.
\end{equation*}
\end{proof}

\begin{remark}
The accuracy of this method hinges on the assumption that the probe displacement $\norm{\point_{i+1} - \point_i}$ accurately reflects the change in normal indentation $\Delta\indentation$. Any tangential motion or significant change in the surface normal between probes introduces error.
\end{remark}


\section{Estimation of Young's Modulus and Curvature from Compliance Variation}

While the two-point method (Proposition~\ref{prop:youngs_modulus}) estimates Young's modulus by assuming a locally linear force-displacement relationship (constant incremental compliance), it cannot determine curvature as the geometric indentation term cancels out. However, if the incremental compliance itself changes with the load level, this variation may contain information about the underlying curvature, particularly if the contact behavior transitions between regimes.

We propose a method based on the hypothesis that contact behaves like Hertzian contact at a contact radius $a < \radius$ at low forces, where the model assumes a sphere of radius $R=1/\curvature$ on a flat surface. From this low-force regime, we transition towards the flat-punch model (contact radius fixed at $a=\radius$) at higher forces \cite[Ch.~5]{Johnson2003a}:

\begin{itemize}
    \item \textbf{Hertzian Regime ($a < \radius$):} Indentation $\indentation \approx C \force^{2/3}$, where $C = (9 \curvature / (16 \estar^2))^{1/3}$. The incremental compliance is $m(\force) = \dd\indentation/\dd\force \approx C (2/3) \force^{-1/3}$.
    \item \textbf{Flat Punch Regime ($a = \radius$):} Additive model $\indentation \approx \force / (2 \estar \radius) + \indentation_{\mathrm{geom}}$. The incremental compliance is $m = \dd\indentation/\dd\force \approx 1 / (2 \estar \radius)$, which is constant.
\end{itemize}
The transition force is approximately $\force_{\mathrm{trans}} \approx (4/3) \estar \radius^3 \curvature$. If probing occurs across this transition, the measured incremental compliance $m_j = \Delta \indentation_j / \Delta \force_j$ should decrease as the force level increases.

\begin{proposition}[Estimation of E and $\curvature$ from Varying Compliance]\label{prop:e_kappa_variation}
Consider a series of $N \ge 3$ probes ($j=1, \dots, N$) applied at the same nominal location with increasing forces $\force_j$. Assume the conditions of Proposition~\ref{prop:youngs_modulus} regarding normal displacement hold for each increment. Let $\nu$ and $\radius$ be known. Assume the lowest forces ($\force_1, \force_2$) correspond to a predominantly Hertzian regime, and the highest forces ($\force_{N-1}, \force_N$) correspond to the flat-punch regime.

\textbf{Procedure:}
\begin{enumerate}
    \item Calculate incremental forces $\Delta \force_j = \force_{j+1} - \force_j$ and displacements $\Delta \indentation_j = \norm{\point_{j+1} - \point_j}$ for $j=1, \dots, N-1$.
    \item Calculate the incremental compliance for each interval: $m_j = \Delta \indentation_j / \Delta \force_j$.
    \item Estimate the compliance in the high-force (flat-punch) regime, $m_{\mathrm{high}}$. This could be $m_{N-1}$ or an average over several high-force intervals if $N$ is large.
    \item Estimate the plane strain modulus $\estar$ using the high-force compliance:
        \begin{equation*}
            \hat{\estar} = \frac{1}{2 \radius m_{\mathrm{high}}}.
        \end{equation*}
        Estimate Young's modulus as $\hat{\youngs} = \hat{\estar}(1 - \poisson^2)$.
    \item Estimate the compliance in the low-force (Hertzian) regime, $m_{\mathrm{low}}$ (e.g., $m_1$). Let $\force_{\mathrm{low}}$ be the average force in this interval, $\force_{\mathrm{low}} = (\force_1 + \force_2)/2$.
    \item Estimate the Hertz constant $C$ by assuming $m_{\mathrm{low}} \approx \dd\indentation/\dd\force|_{\force_{\mathrm{low}}}$ from the Hertzian model:
        \begin{equation*}
            \hat{C} \approx m_{\mathrm{low}} \frac{3}{2} \force_{\mathrm{low}}^{1/3}.
        \end{equation*}
    \item Estimate the curvature $\curvature$ using the Hertz relation $\hat{C}^3 = 9 \hat{\curvature} / (16 \hat{\estar}^2)$ and the previously estimated $\hat{\estar}$:
        \begin{equation*}
            \hat{\curvature} = \frac{16 \hat{C}^3 \hat{\estar}^2}{9}.
        \end{equation*}
        Substituting the expressions for $\hat{C}$ and $\hat{\estar}$:
        \begin{equation*}
        \begin{split}
        	\hat{\curvature} &\approx \frac{16}{9} \left( m_{\mathrm{low}} \frac{3}{2} \force_{\mathrm{low}}^{1/3} \right)^3 \left( \frac{1}{2 \radius m_{\mathrm{high}}} \right)^2 \\
             &= \frac{16}{9} \cdot m_{\mathrm{low}}^3 \cdot \frac{27}{8} \force_{\mathrm{low}} \cdot \frac{1}{4 \radius^2 m_{\mathrm{high}}^2},
        \end{split}
        \end{equation*}
        \begin{equation*}
             \hat{\curvature} \approx \frac{3}{2} \frac{\force_{\mathrm{low}} m_{\mathrm{low}}^3}{\radius^2 m_{\mathrm{high}}^2}.
        \end{equation*}
\end{enumerate}
\end{proposition}

\begin{proofsketch}
The estimation of $\hat{\estar}$ follows directly from equating the high-force incremental compliance $m_{\mathrm{high}}$ with the theoretical constant compliance $1/(2 \radius \estar)$ from the flat-punch model.
The estimation of $\hat{C}$ relies on approximating the incremental compliance $m_{\mathrm{low}} = \Delta \indentation_1 / \Delta \force_1$ at low forces with the derivative $\dd\indentation/\dd\force$ from the Hertzian model $\indentation = C \force^{2/3}$, evaluated at the average force $\force_{\mathrm{low}}$. Specifically, $\dd\indentation/\dd\force = C (2/3) \force^{-1/3}$, which is rearranged to solve for $C$.
The final estimation of $\hat{\curvature}$ uses the definition of the Hertz constant $C = (9 \curvature / (16 \estar^2))^{1/3}$, substituting the previously estimated values $\hat{C}$ and $\hat{\estar}$ and solving for $\curvature$. The algebraic simplification yields the final expression.
\end{proofsketch}

\section{Application}

We now consider an example in which we consider the deformation of a solid sphere of adipose tissue with Young's modulus $E = 8$ kPa and Poisson's ratio $\nu = 0.45$ \cite{Wenderott2020}. The sphere is taken to have a radius of 10 cm, whereas the contact probe (flat punch) is assumed to have a radius of 1 cm. We use two force levels, $P_1 = 3$ N and $P_2 = 1.5 P_1 = 4.5$ N, to reconstruct the Young's modulus of the tissue. In addition, we incorporate zero-mean normally distributed additive noise in both the position and surface normal data with a standard deviation of $\sigma = 10^-3$. We consider random sampling across the surface, with a sample size of 500 samples, where each sample is comprised of a probing action at force $P_1$ and $P_2$.

The reconstructed cavity is shown in Figure~\ref{fig:sphere reconstruction}, showing minimal distortion even under a small sample size. The results of the Young's modulus reconstruction algorithm of Proposition~\ref{prop:youngs_modulus} are shown in Figure~\ref{fig:youngs modulus}. The mean is found to be 7.93 kPa, with a standard deviation of 1.71 kPa.

\begin{figure}[t]
	\centering
	\includegraphics[width=\linewidth]{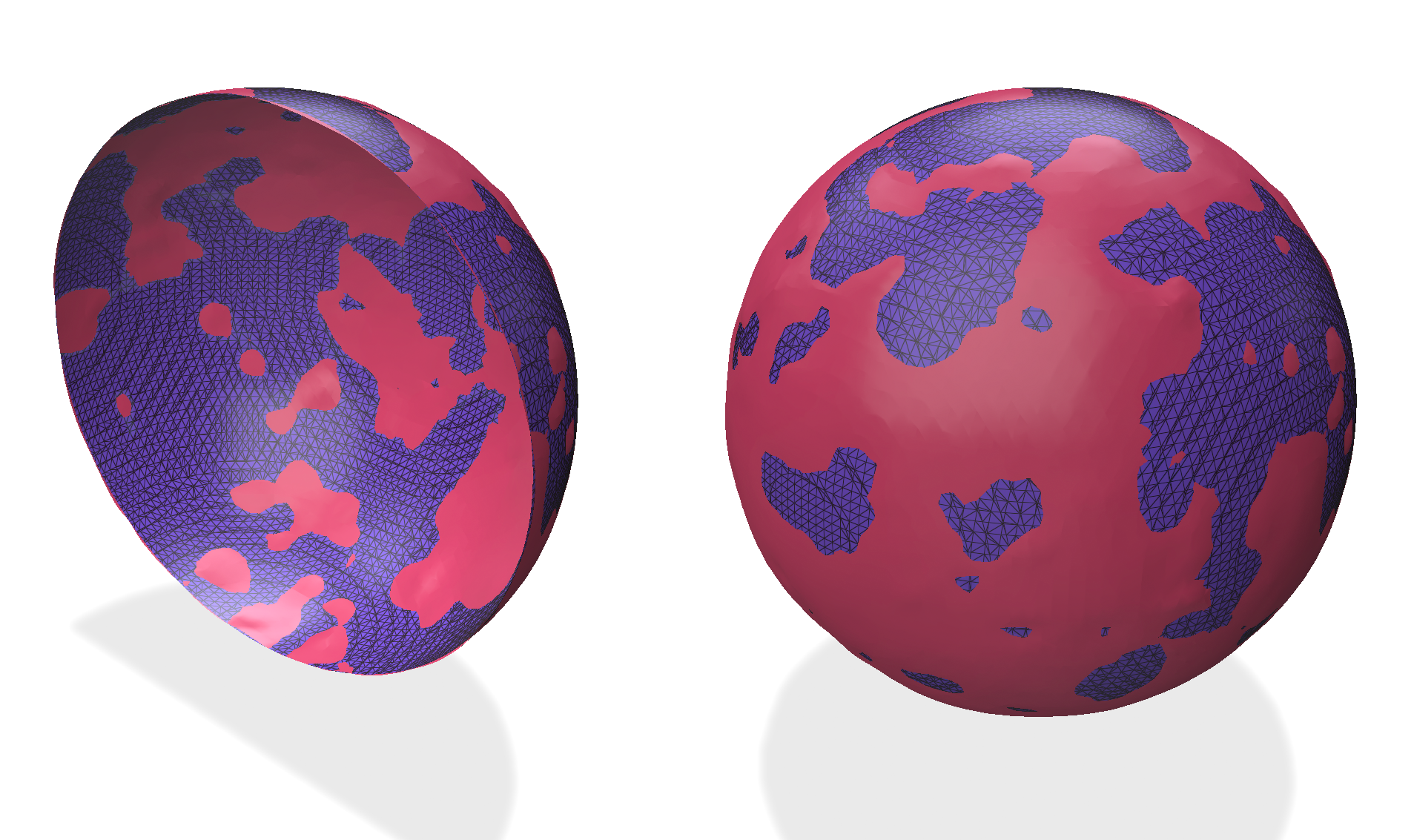}
	\caption{True (purple) and PROD-reconstructed sphere (red) based on 500 samples. Despite additive noise, the sphere is reconstructed accurately.}
	\label{fig:sphere reconstruction}
\end{figure}

\begin{figure}[t]
	\centering
	\includegraphics[width=\linewidth]{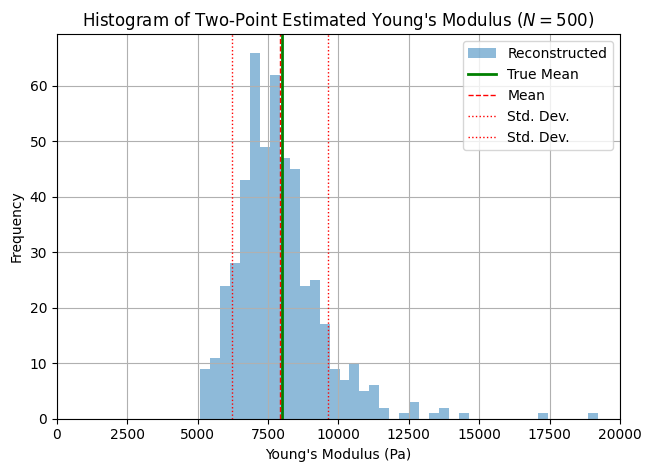}
	\caption{Young's modulus distribution using the two-sample estimation approach of Proposition~\ref{prop:youngs_modulus}. Even at small sample sizes and under noise, a correct mean is found at $7.93$ kPa.}
	\label{fig:youngs modulus}
\end{figure}

\section{Conclusion}

This work introduced PROD, a novel framework for reconstructing both the geometry and mechanical properties of deformable objects using elastodynamic signed distance functions. The proposed method demonstrated robustness in handling pose errors, non-normal force applications, and curvature inaccuracies in simulated scenarios. Outstanding future work includes validating PROD on a robotic platform with ex vivo tissue to assess real-time performance, as well as conducting sensitivity and robustness analyses to evaluate the method's applicability across diverse materials and probing conditions. These extensions will further solidify the practical viability of PROD in robotic manipulation, medical imaging, and haptic feedback systems.

%
%
%
%

\section*{Acknowledgment}

The author would like to thank Prof.~Joseph Bentsman for fruitful discussions that contributed to the realization of this work. This work was inspired by prior efforts sponsored by the National Institute of Biomedical Imaging and Bioengineering of the National Institutes of Health under award number R01EB029766. The content is solely the responsibility of the author and does not necessarily represent the official views of the National Institutes of Health.




%
\bibliographystyle{IEEEtran}
\bibliography{root}

%

%
%
%





\end{document}